\documentclass[preprint,12pt]{article}

\usepackage{algorithm2e}
\usepackage{color,ulem}
\usepackage{graphicx}

\usepackage{amssymb}

\usepackage{amsthm}

\newtheorem{definition}{Definition}

\newtheorem{proposition}{Proposition}

\newtheorem{remark}{Remark}
\newtheorem{theorem}{Theorem}
\newtheorem{proc}{Procedure}

\title{3D Well-composed Polyhedral Complexes}
\author{Rocio Gonzalez-Diaz$^1$,
Maria-Jose Jimenez$^1$,
Belen Medrano
\\
\small Applied Math Department, University of Seville, \\
\small Av. Reina Mercedes, s/n, Seville, Spain\\
\small e-mails: $\{$rogodi,majiro,belenmg$\}$@us.es}

\begin{document}
\maketitle

\begin{abstract}
A binary three-dimensional (3D) image $I$ is well-composed if the boundary surface of its continuous analog is a 2D manifold.
Since 3D images are not often well-composed, there are several voxel-based methods (``repairing" algorithms) for turning them into well-composed ones
but these methods either do not guarantee the topological equivalence between the original image and its corresponding well-composed one or involve sub-sampling the whole image. 
 In this paper, we present a method to locally ``repair" the cubical complex $Q(I)$ (embedded in $\mathbb{R}^3$) associated to $I$ to obtain a
 polyhedral complex $P(I)$ homotopy equivalent to $Q(I)$ such that the boundary of every connected component of $P(I)$ is a 2D manifold.
The reparation is performed via a new codification system for $P(I)$ under the form of a 3D grayscale image that allows an efficient access to cells
and their faces.

\end{abstract}

\section{Introduction}

3D well-composed images \cite{Lat97} enjoy important topological and geometrical properties in such a way that several algorithms used in computer vision, computer graphics and image processing are simpler.
For example, 
thinning algorithms can be simplified and naturally made parallel if the input image is well-composed \cite{Lat95,MAM04}; some algorithms for computing surface curvature or extracting adaptive triangulated surfaces assume that the input image is well-composed \cite{KL04}.
However, our main motivation is that of (co)homology computations on the cell complex representing the 3D image \cite{GJM10,iwcia2011}. We could take advantage of a well-composed-like representation since computations could be performed only on the boundary subcomplex.

 Since 3D images are often not well-composed, there are several methods (repairing algorithms) for turning them into well-composed ones 
\cite{Lat98,SLTGG08},
but these methods do not guarantee the topological equivalence between the original and its corresponding well-composed image. In fact, the purpose may even be to simplify as much as possible the topology (in the sense of removing little topological artifacts). However, we are concerned with the fact of preserving the topology of the input image having in mind cases in which subtle details may be important.

In \cite{SLS07}, the authors provide a solution to the problem of topology preservation during digitization of 3D objects. They use several reconstruction methods that all result in a 3D object with a 2D manifold surface. More specifically, one of the proposed methods is a voxel-based method called Majority Interpolation, by which resolution is doubled in any direction and new sampling voxels are added to the foreground under some constraints. 
Other method is based on the most common reconstruction methods for 3D digital images is the marching cube (MC) algorithm \cite{marchingcube} which analyzes local configurations of eight neighboring sampling points in order to reconstruct a polygonal surface. 
There even exists a MC variant, called  asymmetric marching cubes, which generates the reconstruction of  manifold surfaces (see \cite[page 101]{Ste08}). 
A different approach is made in \cite{LM00}, where the authors create a polyhedral complex as the continuous analog of a set of voxels with given digital adjacencies. They also show that such a continuous analog corresponds to the usual definition of iso-surface in the $3D$ case. 

In our approach, we first consider a  a cubical complex $Q(I)$ associated to a
voxel-based representation of the given image $I$. 
Then, we develop a new scheme of representation, called ExtendedCubeMap (ECM) representation, based on a 3D grayscale image storing the cells and the boundary face relations of cells of $Q(I)$.
 Working on an ECM representation of $Q(I)$, we
 design a procedure to obtain a polyhedral complex $P(I)$
homotopy equivalent to $Q(I)$. The importance of our method is that the cells of   $P(I)$ are totally encoded in a 3D grayscale image $g_P$, 
 and their boundary face relations in a set of structuring elements $B_P$. It is worth to mention that the set $B_P$ remains the same for any polyhedral complex $P(I)$ computed using our method.

In our prequel paper \cite{dgci}, the complex $P(I)$ homotopy equivalent to $Q(I)$, was a cell complex constructed with more general building blocks than polyhedra depending on the local configuration of voxels. In this paper, $P(I)$ is always a polyhedral complex, constructed with a general procedure (that is valid for all the local configurations), with the advantage that it can be stored in a matrix form (a 3D grayscale image) in a way that we do not need to build $P(I)$ to obtain face relations between its cells. 

Section~\ref{polyhedral} is devoted to clarify, first, the correspondence between 3D binary digital images and cubical complexes; the notion of well-composedness is also introduced as well as its extension to complete polyhedral complexes. Section~\ref{representation} describes a new codification system 
called ECM representation
of cubical complexes which is 
also valid for other more general polyhedral complexes
as we will see later in this paper.
 Section~\ref{repairing} describes the repairing algorithm to get a well-composed polyhedral complex, $P(I)$, starting from the cubical complex $Q(I)$ associated to a non-well-composed image $I$.
 The repairing process is, in fact, accomplished on the ECM representation $E_Q$ of $Q(I)$ to get the ECM representation $E_P$ of $P(I)$. 
Finally, we draw some conclusions and statements for future work in the last section.

\section{3D Images, Polyhedral Complexes and Well-composedness}\label{polyhedral}

Consider $\mathbb{Z}^3$ as the set of points with integer coordinates in 3D space $\mathbb{R}^3$.
A {\it 3D binary digital image} (or 3D image for short)
is a set  $I=(\mathbb{Z}^3,26,6,B)$ (or $I=(\mathbb{Z}^3,B)$, for short), where $B\subset \mathbb{Z}^3$ is the {\itshape foreground}, $B^c=\mathbb{Z}^3 \backslash B$ the {\itshape background}, and $(26,6)$ is the adjacency relation for the foreground and background, respectively.
A point of $\mathbb{Z}^3$ can be interpreted as a unit closed cube (called {\it voxel}) in $\mathbb{R}^3$ centered at the point and with faces parallel to the coordinate planes.
The set of voxels centered at the points of $B$ in $\mathbb{R}^3$ is called the {\it continuous analog} of $I$ and it is denoted by $C(I)$. The {\it boundary surface} of  $C(I)$ is the set of points in $\mathbb{R}^3$ that are shared by the voxels centered at points of $B$ and those centered at points of $B^c$ (see \cite{artzy,herman,rosenfeld}).

Recall that a 3D image $I=(\mathbb{Z}^3,B)$ is {\it well-composed} \cite{Lat97} if the {\it boundary surface} of  $C(I)$
 is a 2D manifold,
i.e. each point  has a neighborhood homeomorphic to $\mathbb{R}^2$
 (it ``looks" locally like a planar open set).
The set of voxels of $C(I)$, together with all their faces (squares, edges and vertices) and the coface relationship between them, constitute a combinatorial structure called {\it cubical complex},
denoted by $Q(I)$
whose geometric realization is exactly $C(I)$.
The topology of $Q(I)$ reflects the topology of $I$ whenever $(26,6)-$adjacency is considered on $I$.
This way, we will say that the cubical complex $Q(I)$ associated to a 3D image $I$ is well-composed if the corresponding image $I$ is well-composed.
Considering configurations of $8$ cubes sharing a vertex, one finds
eleven different $8-$cube configurations (modulo reflections and $90-$degree rotations) around a critical vertex,
 as showed in  Fig.~\ref{critical_conf}, where well-composedness condition is not satisfied.
For the eleven $8-$cube configurations, the central vertex $v$ is a {\it critical vertex} in the sense that it does not have a neighborhood in the boundary surface of $Q(I)$,
homeomorphic to $\mathbb{R}^2$.
 These eleven $8-$cube configurations come exactly from the $(2\times 2\times 2)-$configurations which contain one of the just two configurations being presented in \cite[Fig. 3]{Lat97}.
A cubical complex is a specific type of {\it polyhedral complex} (see \cite{koslov}). A  polyhedral complex $K$ is a combinatorial structure by which a space is decomposed into vertices, edges, polygons  and polyhedra (cells, in general) that are glued together by their boundaries such that the intersection of any two cells of the complex is also a cell of the complex.
Notice that the structure of a  polyhedral complex $K$ can be considered as purely combinatorial (i.e., a set of cells with coface relations between them), but in this work, it is associated to a specific geometric realization of $K$ in $\mathbb{R}^3$.

The dimension of an $i-$cell $\sigma\in K$ is $dim(\sigma)=i$. A cell $\mu\in K$ is a {\it face} of a cell $\sigma\in K$ if $\mu$ lies in the boundary of $\sigma$ and $dim(\mu)\leq dim(\sigma)$. The cell $\sigma$ is called a {\it coface} of $\mu$. A cell $\mu$ is {\it maximal} if it is not a face of any other cell $\sigma\in K$ and {\it free} if it is face of exactly one maximal cell of $K$. 
 The {\it boundary subcomplex} $\partial K$ is made up by the free cells of $K$, together with all their faces. 
The {\it closure} of a subset $S$ of $K$ is the smallest  subcomplex of $K$ that contains each cell in $S$. It is obtained by repeatedly adding to $S$ each face of every cell in $S$. The {\it star of $S$} (denoted $St\, S$) is the set of all cells in $K$ that have any faces in $S$.  
Note that the star is generally not a cell complex itself. 
The {\it link of $S$} (denoted by $Lk\, S$)  is the closure of the star of $S$ minus the star of all faces of $S$.
$K$ is an $n$D polyhedral complex if for any cell $\sigma\in K$, $dim(\sigma)\leq n$ and for some maximal cell $\mu\in K$, $dim(\mu)=n$.
An $n$D polyhedral complex is {\it complete} if all its maximal cells have dimension $n$.

Observe that if $K$ is  complete, so it is $\partial K$.
In particular, the cubical complex associated to an image $I$, $Q(I)$, and its boundary subcomplex, $\partial Q(I)$, are, respectively, 3D and 2D complete polyhedral complexes.

\begin{figure}[t!]
	\centering
		\includegraphics[width=13.7cm]{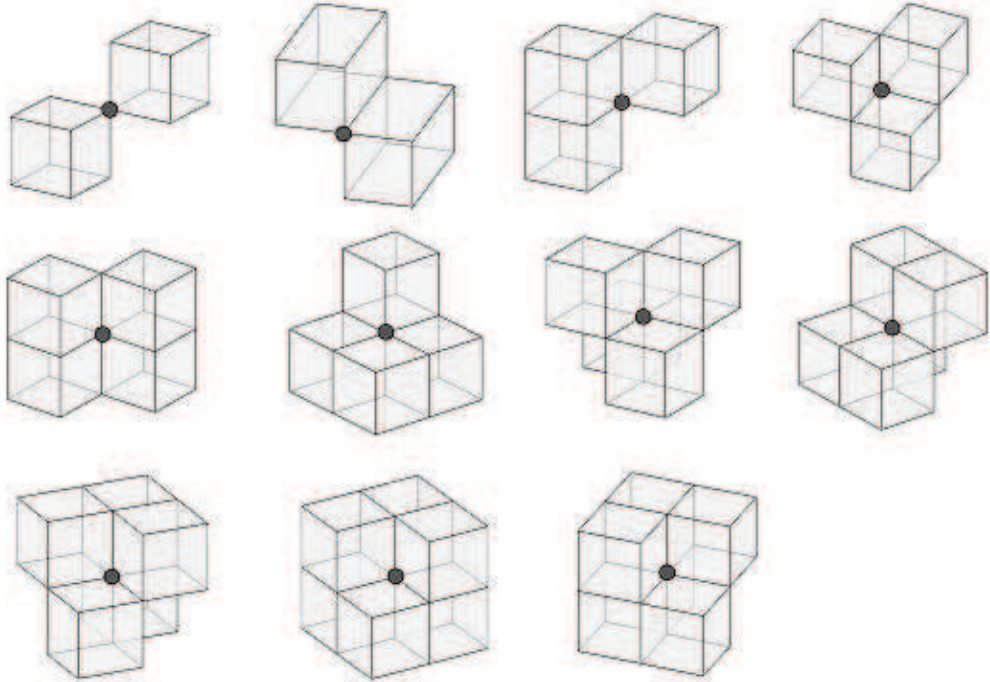}
	\caption{The eleven $8-$cube configurations around a critical vertex.}
	\label{critical_conf}
\end{figure}

The following definition extends the notion of well-composedness to 3D complete polyhedral complexes.

\begin{definition}\label{extended-well-composed}
A 3D complete polyhedral complex $K$ embedded in $\mathbb{R}^3$ is {\it well-composed} if 
the subcomplex $\partial K$ is
a disconnected union of 2D manifolds.
\end{definition}

\begin{proposition}\label{E1E2} A 3D complete polyhedral complex $K$ is  well-composed if the subcomplex $\partial K$ satisfies that:
\begin{itemize}
\item [(E1)] any edge has exactly two $2$-cofaces in $\partial K$;
\item [(E2)] for any vertex $v\in\partial K$, $Lk\,\{v\}$ in $\partial K$ has exactly one connected component.
\end{itemize}
\end{proposition}
\begin{proof}
Take a point $p\in \partial K$. Recall that $\partial K$ is made up by a set of $2-$cells, together with all their faces. If $p$ lies  inside a $2-$cell in $\partial K$, obviously the proposition holds; if $p$ lies on an edge $e\in\partial K$, condition $E1$ applies and hence the proposition holds, too; if $p$ is a vertex of $\partial K$, the star of $p$ in $\partial K$, is homeomorphic to a disk (thanks to conditions $E1$ and $E2$).
\end{proof}
Notice that conditions $E1$ and $E2$ are not satisfied (simultaneously) in the configurations showed in Fig.~\ref{critical_conf}.
\begin{definition}
Given a 3D complete polyhedral complex $K$ embedded in $\mathbb{R}^3$, a vertex $v$ is {\it critical} if either $E1$ fails for some $1-$coface of $v$ or $E2$ is not satisfied by $v$. Then, $K$ is well-composed if there is not any critical vertex in $\partial K$.
\end{definition}

An alternative definition of well-composedness has been presented in \cite[page 11]{Ste08} based on adjacency of maximal cells. In this paper, we use the one involving the notion of critical vertex given above since the repairing process is done on that vertices and their stars.

\section{ExtendedCubeMap Representation}\label{representation}

In \cite{chao}, the authors presented a data-structure (called \textit{CubeMap representation}) designed to compactly store and quickly manipulate cubical complexes. A similar structure was first introduced in CAPD library \cite{capd} for computing cubical homology. In that representation,
 for each cube (voxel), vertices, edges, squares and the cube itself were encoded in a $3\times 3\times 3$ array with values representing the dimension of each cell (see Fig.~\ref{digimagecell3D}.c). For a set of voxels, the corresponding array is composed by overlapping copies of arrays of size $3\times 3\times 3$. 
The CubeMap representation is, in fact, a 3D generalization of the Khalimsky grid, which has already been introduced and extensively used by V. Kovalevsky (e.g. in \cite{Kov08}).

Inspired by that representation scheme, we define a new data-structure, called ExtendedCubeMap representation, that allows to store not only 3D cubical complexes but also the new polyhedral complexes that will be constructed. This new codification is still presented under a 3D array form, also encoding the dimension of the cells that are represented (that is, a 3D grayscale image). In such a structure, the information of boundary relations between represented cells will be extracted by searching for certain structuring elements inside the representation.

An {\it $n$D grayscale image} is a map $g:\mathbb{Z}^n\to \mathbb{Z}$. Given a point $p\in \mathbb{Z}^n$, $g(p)$ is referred to as the {\it color} of $p$. 
A {\it structuring element} is also an $n$D grayscale image $b:D_b\subseteq \mathbb{Z}^n\to \mathbb{Z}$ 
whose domain contains the origin 
$o\in D_b$. 
A structuring element will be  used to perform a given operation around a certain neighborhood of a point.

\begin{definition}\label{ECM}
Given an $n$D complete polyhedral complex $K$,  an {\it ExtendedCubeMap
(ECM)} representation of $K$ is a triple $E_K=(h_K,g_K,B_K)$ where:
\begin{itemize}
\item 
 $h_K: D_K\to K$ 
is a 
bijective
 function,
for a certain domain $D_K\subset \mathbb{Z}^n$ with as many points as cells in $K$. For each cell $\sigma$, denote by $p_\sigma$ the point $h^{-1}_K(\sigma)$ representing $\sigma$;
\item $g_K:\mathbb{Z}^n\to \{-1,0,1,\dots,n\}$ is an $n$D grayscale image,
such that:
\begin{itemize}
\item $g_K(p)=dim(h_K(p))$, for any $p\in D_K$, .
\item $g_K(p)=-1$,  For any $p\in \mathbb{Z}^n\setminus D_K$.
\end{itemize}
\item $B_K$ is a set of structuring elements $\{b:D_b\subset\mathbb{Z}^n \to \mathbb{Z}\}$
such that for any $i-$cell $\sigma\in K$, there exists a single structuring element $b_{\sigma}: D_{\sigma}\to\mathbb{Z}$ in $B_K$ such that:
\begin{itemize}
\item $b_\sigma(p)=g_K(p+p_\sigma)$,  for any $p\in D_\sigma$. In particular, $b_{\sigma}(o)=i$;
\item If $p\in D_{\sigma}$ and $p\neq o$ then either $b_{\sigma}(p)=i-1$ or $b_{\sigma}(p)=-1$;
\item 
A point $p\in \mathbb{Z}^3$ 
satisfies that $p-p_{\sigma}\in D_\sigma$ and $b_\sigma(p-p_{\sigma})=i-1$ if and only if $p$ represents a cell $\mu$
(that is, $p=h_K^{-1}(\mu)$) which is an $(i-1)-$face of $\sigma$.
\end{itemize}
\end{itemize}
\end{definition}

\begin{remark}\label{colors}
In order to visualize examples of ECM representations, from now on, a point $p\in \mathbb{Z}^3$ is colored black if $g_Q(p)=3$, green if $g_Q(p)=2$, red if $g_Q(p)=1$, blue if $g_Q(p)=0$ and white if $g_Q(p)=-1$. 
For a better understanding of the pictures, the reader is referred to the web version of this paper. \end{remark}

An ECM representation of $K$ may not exist and may not be unique (see, for example, Fig.~\ref{digimagecell3D}.(c,d) which shows 
two different ECM representations of a single cube together with all its faces). Nevertheless, given a cubical complex $Q(I)$ and the ECM representation $E_Q$ of $Q(I)$ given in Prop.~\ref{EQ}, the procedures described in the next section will provide a unique polyhedral complex $P(I)$ as well as a unique ECM representation $E_P$ of $P(I)$.

\begin{figure}[t!]
	\centering
		\includegraphics[width=12cm]{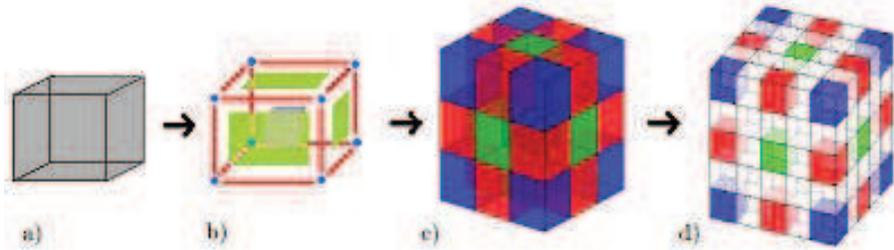}
	\caption{a) A single voxel $I$; b) the  cubical complex $Q(I)$ associated to $I$; c) 
The color values in the CubeMap representation of $Q(I)$ which is also an ECM representation of $Q(I)$; d) 
The color values in the ECM representation $E_Q$ of $Q(I)$. In both c) and d) the central voxel is colored black.}
	\label{digimagecell3D}
\end{figure}

Since the map $h_K$ is bijective, each cell $\sigma\in K$ is represented by a single point $p_\sigma=h^{-1}_K(\sigma) \in \mathbb{Z}^n$. 
Moreover, 
for each $p\in D_K$, $g_K(p)$ codifies the dimension of the cell in $K$ that $p$ represents (that is, the cell $\sigma=h_K(p)$). 
Finally, observe that the structuring element $b_\sigma:D_{\sigma}\to \mathbb{Z}$ associated to a cell $\sigma \in K$ provides a codification of the boundary face relations for $\sigma$.
Therefore,  within this codification system, the extraction of boundary faces of any cell $\sigma$ can be done by checking 
if any of the  structuring elements of $B_K$ fits around the point $p_{\sigma}$
(that is, the point in $D_K$ that represents $\sigma$). This last operation could be seen as a morphological erosion of the 3D grayscale image $g_K$ by the corresponding structuring element.

\begin{definition}
Given an ECM representation $E_K=(h_K,g_K,B_K)$ of a polyhedral complex $K$, a structuring element $b: D_b\to \mathbb{Z}$ is said to {\it fit around a point} 
$p\in D_K\subset \mathbb{Z}^n$ if, for any $q\in D_b$, $g_K(p+q)=b(q)$.
\end{definition}

Fig.~\ref{structelem} shows an example of a polyhedral complex $K$, 
the color values  in an ECM representation $E_K$ of $K$
 and the structuring elements (modulo $90-$degree rotations) associated to the cells of $K$.

\begin{figure}[t!]
	\centering
		\includegraphics[width=3.5cm]{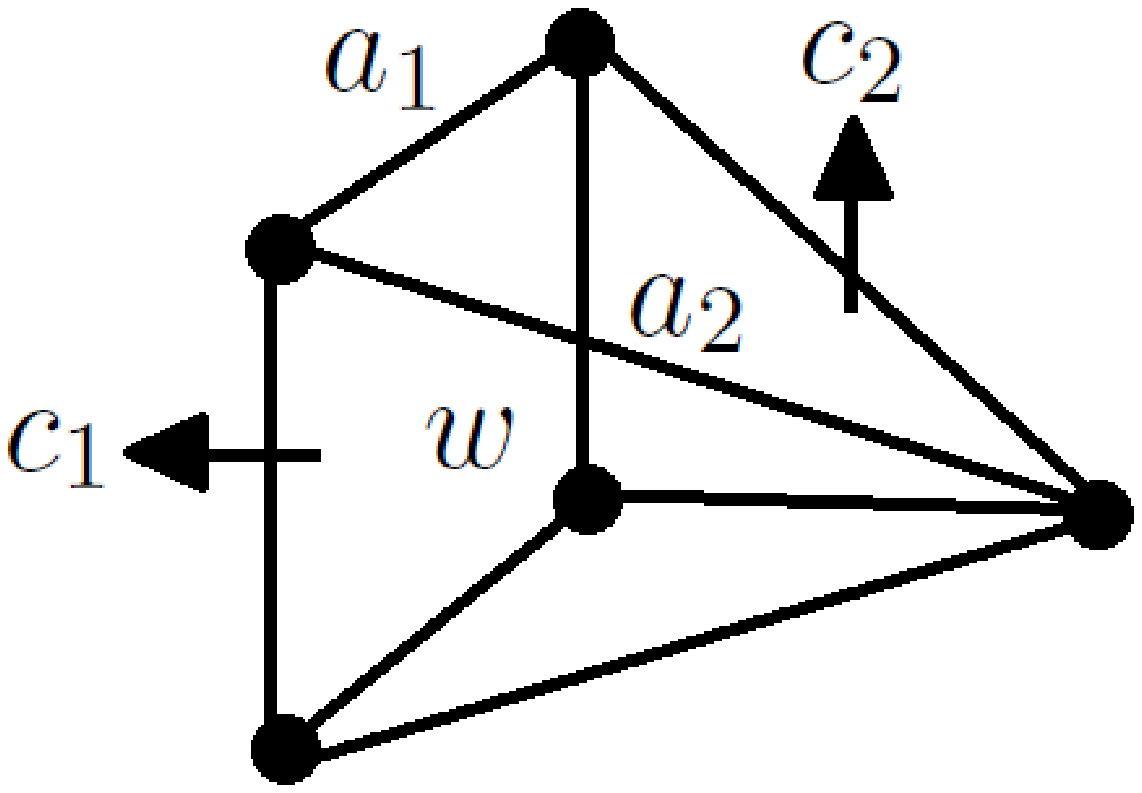}\\[3mm]
    \includegraphics[width=13.5cm]{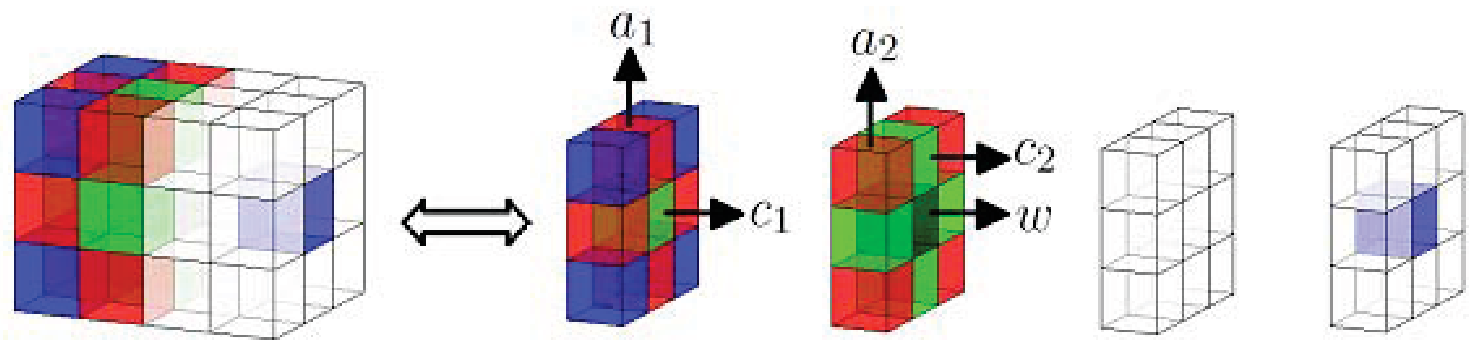}\\[3mm]
    \includegraphics[width=13.5cm]{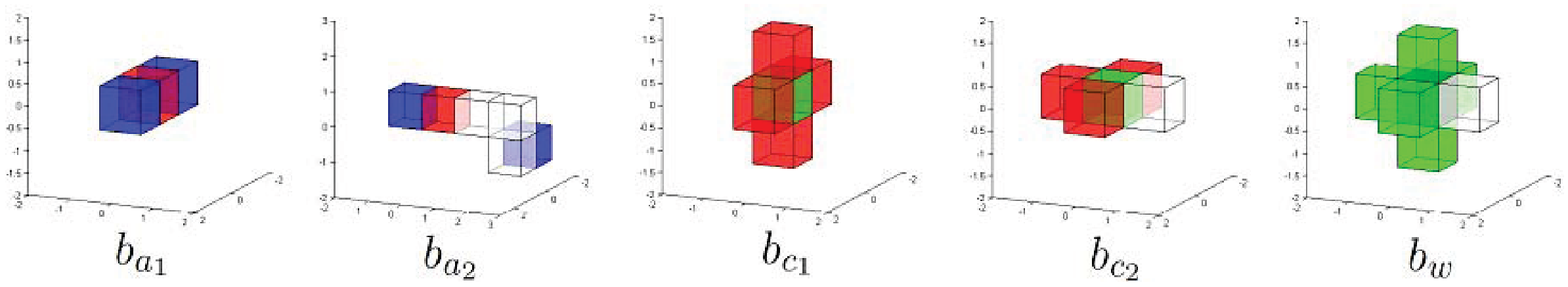}
	\caption{Top: A polyhedral complex $K$ (a solid pyramid). Middle:
The color values of the voxels in an ECM representation of $K$. 
Bottom: Structuring elements 
(modulo $90-$degree rotations) 
 associated to the cells of $K$.}
	\label{structelem}
\end{figure}

 The following result provides an ECM representation of a cubical complex $Q(I)$ associated to a 3D image $I$. Let $r_{\sigma}$ be the barycenter of $\sigma\in Q(I)$ and $D_Q=\{p\in \mathbb{Z}^3$ such that
$p=4 r_{\sigma}$ for some $\sigma\in Q(I)\}$. 

\begin{proposition}\label{EQ}
Let $Q(I)$ be the cubical complex associated to a 3D image $I$, the triple $E_Q=(h_Q,g_Q,B_Q)$ described below is an ECM representation of $Q(I)$.
\begin{itemize}
\item  $h_Q:D_Q\to Q(I)$, where:
\begin{itemize}
\item $D_Q=\{p\in \mathbb{Z}^3$ such that $p=4 r_{\sigma}$ for some $\sigma\in Q(I)\}$;
\item $h_Q(4 r_{\sigma})=\sigma$ for $\sigma\in Q(I)$;
\end{itemize} 
\item $g_Q:\mathbb{Z}^3\to \{-1,0,1,2,3\}$ is given by 
$g_Q(p)=dim (\sigma)$ if $p=4 r_{\sigma}\in D_Q$,
and  $g_Q(p)=-1$ otherwise. See Fig.~\ref{digimagecell3D}.d;
\item $B_Q=\{b_{\ell}:D_{\ell}\to \mathbb{Z}\}_{\ell=1,2,3}$,
where $b_{\ell}:D_{\ell}\to \mathbb{Z}$ (modulo $90-$degree rotations) is given by:
\begin{itemize}
\item  $D_1=\{o,(\pm 1,0,0),(\pm 2,0,0)\}$, $b_1(o)=1$, 
$b_1(p)=0$ for $p=(\pm 2,0,0)$ and $b_1(p)=0$ otherwise.
 See Fig.~\ref{rojo-azul}.x;
\item $D_2=\{o,(\pm 1,0,0),(0,\pm 1,0),(\pm 2,0,0),(0,\pm 2,0)\}$, $b_2(o)=2$,\\
 $b_2(p)=1$ for $p=(\pm 2,0,0),(0,\pm 2,0)$ and $b_2(p)=-1$ otherwise.
See Fig.~\ref{verde-rojo}.x; 
\item $D_3=
\{o,(\pm 1,0,0)$, $(0,\pm 1,0)$, $(0,0,\pm 1)$, $(\pm 2,0,0),(0,\pm 2,0)$,\\ $(0,0,\pm 2)\}$,
$b_3(o)=3$,  $b_3(p)=2$ for $p=(\pm 2,0,0),(0,\pm 2,0)$, $(0,0,\pm 2)$ and $b_3(p)=-1$ otherwise. See Fig.~\ref{negro-verde}.x.
\end{itemize}
\end{itemize} 
\end{proposition}

\begin{figure}[t!]
	\centering
		\includegraphics[width=10cm]{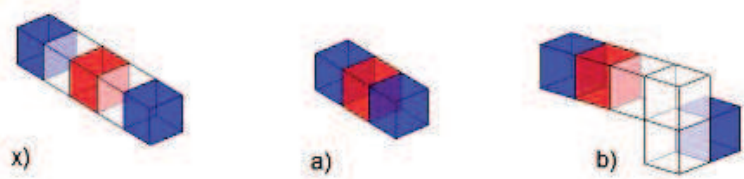}
	\caption{Structuring elements that provide boundary face relations between edges and their $0-$faces.}
	\label{rojo-azul}
\end{figure}

\begin{figure}[t!]
	\centering
		\includegraphics[width=13.5cm]{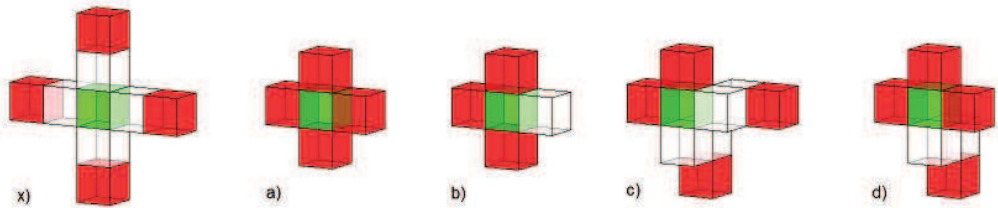}
	\caption{Structuring elements that provide boundary face relations between polygons and their $1-$faces.}
	\label{verde-rojo}
\end{figure}

\begin{figure}[t!]
	\centering
		\includegraphics[width=10cm]{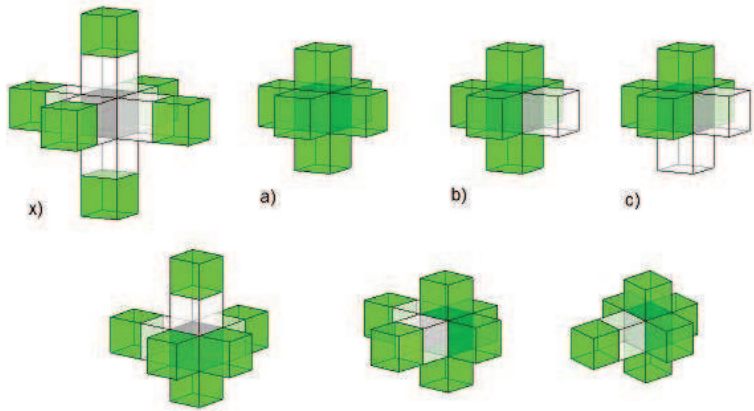}
	\caption{Structuring elements that provide boundary face relations between polyhedra and their $2-$faces.}
	\label{negro-verde}
\end{figure}

Observe that $Q(I)$ is constructed by adding
unit cubes $\sigma$ centered at points $r_{\sigma}=(i,j,k)\in \mathbb{Z}^3$. 
Therefore,  the possible coordinates of $r_{\mu}$ for the $\ell$-faces $\mu$ of $\sigma$ are: 
\begin{itemize}
\item $(i\pm\frac{1}{2},j\pm\frac{1}{2},k\pm\frac{1}{2})$ if $\ell=0$;
\item $(i,j\pm\frac{1}{2},k\pm\frac{1}{2})$, 
$(i\pm\frac{1}{2},j,k\pm\frac{1}{2})$ or 
$(i\pm\frac{1}{2},j\pm\frac{1}{2},k)$, 
 if $\ell=1$;
\item $(i\pm\frac{1}{2},j,k)$, 
$(i,j\pm\frac{1}{2},k)$ or
$(i,j,k\pm\frac{1}{2})$, 
 if $\ell=2$;
\end{itemize}

\begin{figure}[t!]
	\centering
		\includegraphics[width=13.7cm]{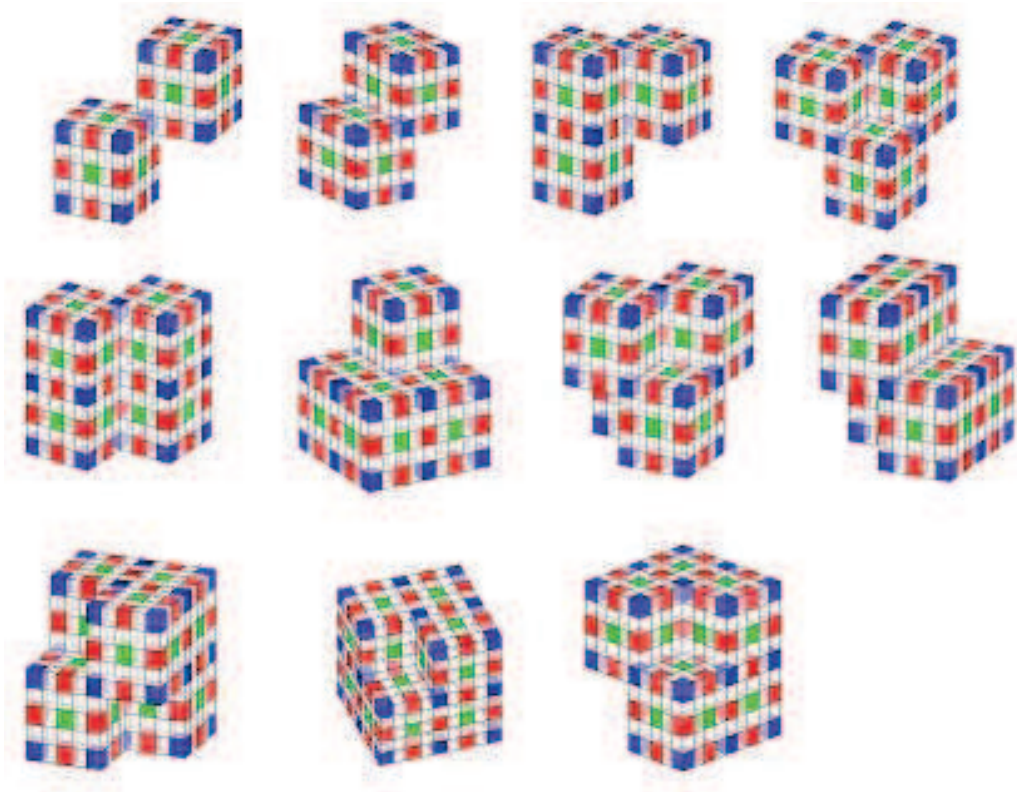}
	\caption{ 
The color values  in the ECM representations $E_Q$
 for the 3D complete cubical complexes  corresponding to the eleven $8-$cube configurations around a critical vertex showed in Fig.~\ref{critical_conf}.}
	\label{criticalpointsEQ}
\end{figure}

\begin{proof}
Since  for any $\sigma\in Q(I)$, $r_{\sigma}$ has either integer or multiple-of$-\frac{1}{2}$ coordinates, then $p=4 r_{\sigma}\in \mathbb{Z}^3$. More specifically,
\begin{itemize}
\item $p$ represents a cube in $Q(I)$, if and only if $p=(4i,4j,4k)$;
\item $p$ represents a square face in $Q(I)$, if and only if either $p=(4i+2,4j,4k)$, $p=(4i,4j+2,4k)$ or $p=(4i,4j,4k+2)$;
\item $p$ represents an edge in $Q(I)$, if and only if either $p=(4i+2,4j+2,4k)$,  $p=(4i,4j+2,4k+2)$ or $p=(4i+2,4j,4k+2)$;
\item $p$ represents a vertex in $Q(I)$, if and only if $p=(4i+2,4j+2,4k+2)$;
\end{itemize}
where $i,j,k\in \mathbb{Z}$.
Besides, $g_Q$, by definition, relates each point $p\in D_Q$ representing a cell $\sigma\in Q(I)$ with the dimension of $\sigma$.
Observe that $b_{\ell}$, $\ell=1,2$, and their $90-$degree rotations describe six structural elements. For each $\ell-$cell $\sigma\in Q(I)$, $\ell=1,2$, only one of the three $b_{\ell}$ is associated to $\sigma$, depending on which axis the
cell (edge or square face) is parallel to. 
$b_3$ is associated to the $3-$cells of $Q(I)$. 
Now, we have to prove that 
for each cell $\sigma \in Q(I)$, there exists a single structuring element $b_{\sigma}\in B_Q$  satisfying the conditions in Def.~\ref{ECM}. 
We will prove this just for an edge parallel to the $x-$axis, since the proof for the rest of the cells is analogous.
 Suppose that $\sigma=e(u,v)$ is a unit edge with endpoints $u=(i-\frac{1}{2},j+\frac{1}{2},k+\frac{1}{2})$ and $v=(i + \frac{1}{2},j+\frac{1}{2},k+\frac{1}{2})$ (where $i,j,k\in \mathbb{Z}$). Then, $p_{\sigma}=(4i,4j+2,4k+2)$,
$p_{u}=(4i-2,4j+2,4k+2)$ and $p_{v}=(4i+2,4j+2,4k+2)$. Therefore, $b_\sigma = b_1: D_1\to\mathbb{Z}$ since:
    \begin{itemize}
    \item $b_1(o)=1=g_Q(p_\sigma)$,\\
$b_1((\pm 1,0,0))=g_Q((4i\pm 1,4j+2,4k+2))=-1$\\
$b_1((\pm 2,0,0))=g_Q((4i\pm 2,4j+2,4k+2))=0$.
       \item If $p=p_\mu$ represents an $(i-1)-$face of $\sigma$, then either $\mu=u$,  and $b_{1}(p_u-p_\sigma)=b_{1}((-2,0,0))=0$; or $\mu=v$ and $b_{1}(p_v-p_\sigma)=b_{1}((2,0,0))=0$.\\
 Besides, if $p \in \mathbb{Z}^3$ such that $b_{1}(p-p_\sigma)=0$, then $p-p_\sigma =(\pm 2,0,0)$ and therefore $p=p_u$ or $p=p_v$.
    \end{itemize}
\end{proof}

See Fig.~\ref{criticalpointsEQ} as an example of the color values in the ECM representations $E_Q$ for several 3D complete cubical complexes $Q(I)$.

\begin{remark}\label{remark:coface}
The coface relations between cells can also be codified in terms of structuring elements in the ECM representation. More specifically, given a vertex $v\in Q(I)$ and its corresponding point $p_v$ in the ECM representation $E_Q$, all the points representing the cofaces of $v$ (see Fig.~\ref{grid}) can be found using a structuring element in the set $B^c_Q=\{b_{\ell}^c:D_{\ell}^c\to \mathbb{Z}\}_{\ell=1,2,3}$, 
where $D_{\ell}^c$ and $b_{\ell}^c:D_{\ell}^c\to \mathbb{Z}$ (modulo $90-$degree rotations) are given by:
\begin{itemize}
\item $D_1^c=\{o,(1,0,0),(2,0,0)\}$; \\
$b_1^c(o)=0$, $b_1^c((1,0,0))=-1$ and  $b_1^c((2,0,0))=1$;
\item $D_2^c=\{o,(1,1,0),(2,2,0)\}$;\\
 $b_2^c(o)=0$, $b_2^c((1,1,0))=-1$ and  $b_2^c((2,2,0))=2$.
\item $D_3^c=\{o,(1,1,1),(2,2,2)\}$;\\
 $b_2^c(o)=0$, $b_2^c((1,1,1))=-1$ and  $b_2^c((2,2,2))=3$.
\end{itemize}
\end{remark}

\begin{figure}[t!]
	\centering
		\includegraphics[width=13cm]{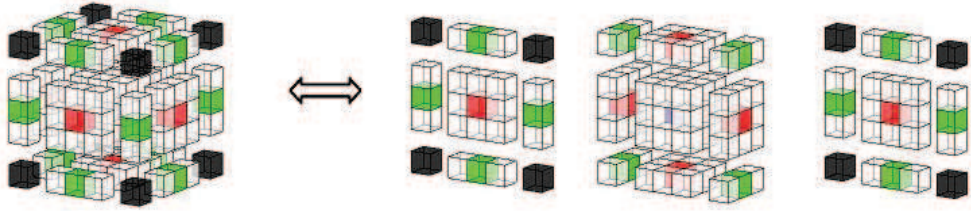}
	\caption{Color values of the voxels representing the cofaces of a vertex $v$ when $v$ is shared by $8$ cubes.}
	\label{grid}
\end{figure}

\begin{proposition}\label{prop:critical}
Given the ECM representation $E_Q$ of a cubical complex $Q(I)$ associated to a binary 3D digital image $I$, 
consider  the set of structuring elements $B^{critical}_Q=\{b^{critical}:D^{critical}\to \mathbb{Z}\}$ 
showed in Fig.~\ref{criticalpoints} 
(modulo reflections and $90-$degree rotations) and described below: 
\begin{itemize}
\item $D^{critical}=\{o,(\pm 1,\pm 1,\pm 1),(\pm 2,\pm 2,\pm 2)\}$;
\item $b^{critical}(o)=0$, $b^{critical}((\pm 1,\pm 1,\pm 1))=-1$ and  
$b^{critical}((\pm 2,\pm 2,\pm 2))$ is either $3$ or $-1$ (see Fig.~\ref{criticalpoints}). 
\end{itemize}
A point $p\in \mathbb{Z}^3$ represents a critical vertex in $Q(I)$ if and only if a structuring element $b^{critical}\in B^{critical}_Q$ fits around $p$.
\end{proposition}

\begin{proof} 
Consider the $8-$cube configuration around a vertex showed in upper-left corner of Fig.~\ref{critical_conf}. 
Suppose that a vertex  
$v=(i+\frac{1}{2},j+\frac{1}{2},k+\frac{1}{2})$ is shared by exactly two unit cubes $c_1$ and $c_2$ 
such that
$r_{c_1}=(i,j,k)$ and $r_{c_2}=(i+1,j+1,k)$, where $i,j,k\in\mathbb{Z}$. Then 
$p_v=h^{-1}_Q(v)=(4i+2,4j+2,4k+2)$,
 $p_{c_1}=h^{-1}_Q(c_1)=(4i, 4j, 4k)$
and $p_{c_2}=h^{-1}_Q(c_2)=(4i+ 4, 4j+4, 4k)$.
 Now consider the structuring element $b_1^{critical}$ depicted in  Fig.~\ref{criticalpoints}.(upper-left). 
Points in $D_1^{critical}$ for which $b_1^{critical}$ takes value $3$ are $p_1=(-2,-2,-2)$ and $p_2=(2,2,-2)$. Let us check that $b_1^{critical}$ fits around $p_v$:
\begin{itemize}
\item $b_1^{critical}(o)=0=g_Q(p_v)$;
\item  $b_1^{critical}(p_1)=3=g_Q(4i,4j,4k)$;
$b_1^{critical}(p_2)=3=g_Q(4i+4,4j+4,4k)$;
\item  $b_1^{critical}(p)=-1=g_Q(p+p_v)$, 
 for  $p\in \{(\pm 1,\pm 1,\pm 1), (\pm 2,\pm 2,\pm 2)\}\setminus \{p_1,p_2\}$, since
 $\frac{1}{4}(p+p_v)$ does not correspond to the barycenter of any cell in $Q(I)$.
\end{itemize}
The rest of the configurations can be proven using similar arguments than above.
\end{proof}

\begin{figure}[t!]
	\centering
		\includegraphics[width=13.7cm]{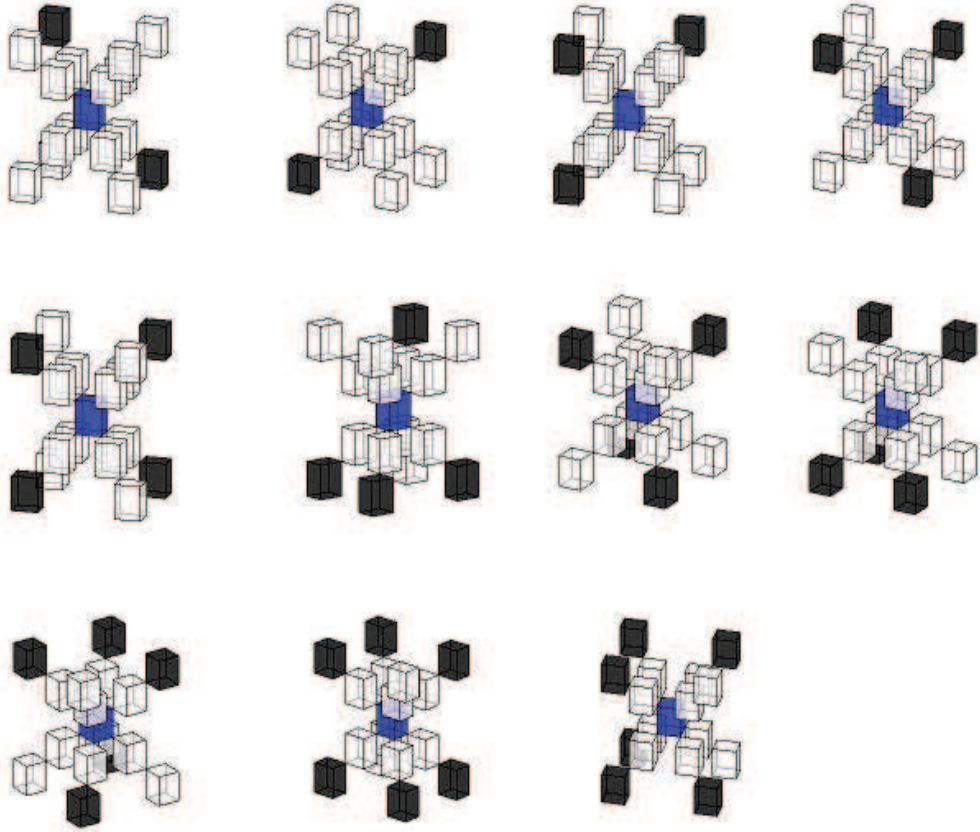}
	\caption{Structuring elements  (modulo reflections and $90-$degree rotations) 
for the eleven $8-$cube  configurations around a critical vertex in $Q(I)$
 showed in Fig.~\ref{critical_conf}.}
	\label{criticalpoints}
\end{figure}

\section{Repairing the Critical Vertices on the ECM Representation $E_Q$ of $Q(I)$
}\label{repairing}

 In this section, given the  cubical complex $Q(I)$ (associated to a 3D image $I$) 
and the ECM representation $E_Q$ of $Q(I)$ from Prop.~\ref{EQ}, 
 we show how to obtain the ECM representation $E_P=(h_P,g_P,B_P)$ of 
a well-composed polyhedral complex $P(I)$ homotopy equivalent to $Q(I)$. 
The ``repairing" process that will be explained below consists of running over all the points $p\in \mathbb{Z}^3$ that represent critical vertices $v\in Q(I)$ and accomplishing a color-changing operation by
modifying the values of $g_Q$ in a neighborhood of $p$. 
New set of structuring elements, $B_P$,  is then defined to codify the boundary relations between cells in $P(I)$.

Let $\ell=1,2$. The following sets of points in $\mathbb{Z}^3$ will be useful for different descriptions of points inside the ECM representation $E_P$:
\begin{itemize}
\item $N_6^\ell=
 \bigcup_{i=1,2,3}\{(x_1,x_2,x_3)\in \mathbb{Z}^3$ s.t. $|x_i|= \ell$ and $x_j=0$ for $j\neq i\}$.
\item $N_{12}^\ell= 
\bigcup_{i=1,2,3}\{(x_1,x_2,x_3)\in \mathbb{Z}^3$ s.t.  $x_i=0$ and $|x_j|= \ell$ for $j\neq i\}$.
\item $N_{8}^{\ell}=
\{(x_1,x_2,x_3)\in \mathbb{Z}^3$ s.t. $|x_i|= \ell$ for $i=1,2,3\}$.
\item $N^{\ell}=
\bigcup_{i=1,2,3}\{(x_1,x_2,x_3)\in \mathbb{Z}^3$ s.t. $|x_i|= \ell$ and $|x_j|\leq \ell$ for $j\neq i\}.$
\item $N^{\leq \ell}= 
 \{(x_1,x_2,x_3)\in \mathbb{Z}^3$ s.t. $|x_i|\leq \ell, i=1,2,3\}$.
\end{itemize}
Finally, for a set $N\subset \mathbb{Z}^3$ and a point $p\in \mathbb{Z}^3$, let us denote by  $N(p)$ the set $\{q+p$ such that $q\in N\}$.

Observe that $N^{\leq 1}=\{o\}\cup N^{1}$ and $N^{\leq 2}=\{o\}\cup N^{1}\cup N^2$.
In fact, given a point $p\in\mathbb{Z}^3$, $N^{\leq \ell}(p)$ is the $(2\ell+1)\times (2\ell+1)\times (2\ell+1)$ block of voxels in $\mathbb{Z}^3$ centered at point $p$. If we consider this block as a 
big cube $C_p$ composed by $(2\ell+1)^3$ unit cubes, then:
\begin{itemize}
\item $N^{\ell}(p)$ are the faces of such a cube $C_p$.
\item $N^{\ell}_6(p)$ are the $6$ endpoints of the $3$ segments with mid-point in $p$, length equal to $2\ell+1$, and parallel to the coordinate axes.
\item $N^{\ell}_{12}(p)$ are the (totally $12$) vertices of the $3$ squares centered at $p$, edge-length equal to $2\ell+1$, parallel to the coordinate planes, and whose edges are parallel to coordinate axes.
\item $N^{\ell}_8(p)$ are the vertices of the cube $C_p$.
\end{itemize}

The following remark ensures that given a 
point $p\in\mathbb{Z}^3$ such that $g_Q(p)=0$ (which corresponds to a vertex $v=h_Q(p)\in Q(I)$),
the modification of the map $g_Q$ on the points in $N^{\leq 2}(p)$
 only affects to points that either represent    cofaces of $v$ or do not represent any cell in $Q(I)$.

\begin{remark}\label{remark:vertice}
Let $p\in\mathbb{Z}^3$ such that $g_Q(p)=0$. Let $v=h_Q(p)$ be the corresponding vertex in $Q(I)$.
\begin{itemize}
\item If  $q\in N^{\leq 2}(p)\setminus (N^2_6(p)\cup N^2_{12}(p)\cup N^2_{8}(p) )$ then $g_Q(q)=-1$ 
\item  If $q\in N_6^2(p)$, then  either $g_Q(q)=1$  ($q$ represents a $1-$coface of $v$)  or $g_Q(q)=-1$.\\ Analogously, 
if $q\in N_{12}^2(p)$, then  either $g_Q(q)=2$  
 ($q$ represents a $2-$coface of $v$)
or $g_Q(q)=-1$.\\
Finally, if
$q\in N_8^2(p)$, then  either $g_Q(q)=3$ 
 ($q$ represents a $3-$coface of $v$)
or $g_Q(q)=-1$.
\end{itemize}
\end{remark}
 See, for example,  Fig.~\ref{grid}, which is an example of the values of $g_Q$ on the set $N^{\leq 2}(p)$ where $p$ represents a vertex shared by eight cubes in $Q(I)$.

Now, the following  method is performed on the ECM representation $E_Q$ 
to construct a new map $g_P:\mathbb{Z}^3\to \mathbb{Z}$.

\begin{proc}\label{methodExtended}
Start with the ECM representation  $E_Q=(h_Q,g_Q,B_Q)$  of a given cubical complex $Q(I)$.
\begin{enumerate}
\item[0.]
Initially,
$g_P(p)=g_Q(p)$ for any $p\in \mathbb{Z}^3$.

\item[1.] 
Compute the set $R\subseteq \mathbb{Z}^3$ where $p\in R$ if $g_Q(p)=0$ 
and one of the structuring elements showed in Fig.~\ref{criticalpoints} fits around $p$ (that is,  
 $h_Q(p)\in Q(I)$ is a critical vertex).
\item[2.] 
For each point  $p\in R$:
\begin{itemize}
\item 
Compute the set $ST_p\subseteq \mathbb{Z}^3$ such that 
$q\in ST_p$ if one of the structuring elements
given in Remark~\ref{remark:coface} fits around $q$ (that is, 
$h_Q(q)$ is a cell in $St\{h_Q(p)\}$).
\item 
$g_P(p)=3$;\\
$g_P(q)=2$ if $q\in N^1_6(p)$;\\
$g_P(q)=1$ if $q\in N^1_{12}(p)$;\\
and $g_P(q)=0$ if $q\in N^1_8(p)$.\\
See Fig.~\ref{Op1}.a.
\item For each point $q\in ST_p$ such that $g_Q(q)=1$, 
let $q_1\in D_Q$ represent the  other endpoint of $q$ different to $p$. Then:\\
$g_P(q)=3$;\\
$g_P(s)=2$ if $s\in N^1_6(q)\setminus (N^1(p)\cup N^1(q_1))$;\\
and $g_P(s)=1$ if $s\in N^1_{12}(q)\setminus (N^1(p)\cup N^1(q_1))$;\\
See Fig.~\ref{Op1}.b.
\item For each point $q\in ST_p$ such that $g_Q(q)=2$, let 
 $q_1,q_2,q_3,q_4\in D_Q$ represent the $1$-faces of the cell represented by $q$.
Then:\\
  $g_P(q)=3$;\\ 
and
$g_P(s)=2$ if $s\in N^1_6(q)\setminus (N^1(q_1)\cup N^1(q_2)\cup N^1(q_3)\cup N^1(q_4))$;\\
See Fig.~\ref{Op1}.c.
\end{itemize}
\end{enumerate}
\end{proc}

Notice that the points $q\in ST_p$ representing cubes in $Q(I)$ remain with the same color value, that is, $g_P(q)=3=g_Q(q)$.
The above procedure could be improved by computing the sets $R$ and $\{ST_p: p\in R\}$ at once, since the coordinates of the different cofaces of a point $p$ representing a critical vertex in $Q(I)$ are known. For the sake of clarity we have separated the seeking of critical vertices (the set $R$) and their cofaces (the set $ST_p$ for each $p$).

\begin{proposition} The map $g_P:\mathbb{Z}^3\to \mathbb{Z}$ is well-defined.
\end{proposition}

\begin{proof}
Observe that the 3D space $\mathbb{Z}^3$ can be decomposed into the following non-overlapping subspaces (see Fig.~\ref{grid}):
$$\begin{array}{rl}
S_0=\{&(4i+2\pm \ell_1,4j+2\pm \ell_2,4k+2\pm \ell_3)\; \}_{i,j,k\in\mathbb{Z},\,\ell_1,\ell_2,\ell_3\in\{0,1\}}\\
S_1=\{&(4i+2\pm \ell_1,4j+2\pm \ell_2,4k),(4i+2\pm \ell_1,4j,4k+2 \pm \ell_2),\\
&(4i,4j+2\pm \ell_1,4k+2\pm \ell_2)\; \}_{i,j,k\in\mathbb{Z},\,\ell_1,\ell_2\in\{0,1\}}\\
S_2=\{&(4i+2\pm \ell,4j,4k),(4i,4j+2\pm \ell,4k ),\\
&(4i,4j,4k+2 \pm \ell)\;\}_{i,j,k\in\mathbb{Z},\,\ell\in\{0,1\}}\\
S_3=\{&(4_i,4_j,4_k)\;\}_{i,j,k\in\mathbb{Z}}\end{array}$$
Then, color-changing process corresponding to 
a point $q\in\mathbb{Z}^3$ with $g_Q(q)=i$, for $i=0,1,2,3$, is performed on $S_i\cap N^{\leq 1}(q)$. In fact, the colors of all the points in $S_i\cap N^{\leq 1}(q)$ are modified. \\
Let $q_1,q_2\in \mathbb{Z}^3$ be points with $i_1=g_Q(q_1)\neq -1$ and $i_2=g_Q(q_2)\neq -1$. The subsets $S_{i_1}\cap N^{\leq 1}(q_1)$ and  $S_{i_2}\cap N^{\leq 1}(q_2)$ never intersect
 since:
\begin{itemize}
\item $N^{\leq 1}(q_1)\cap N^{\leq 1}(q_2)=\emptyset$ if $i_1=i_2$,
\item $S_{i_1}\cap S_{i_2}=\emptyset$ if $i_1\neq i_2$.
\end{itemize}
\end{proof}

\begin{figure}[t!]
	\centering
		\includegraphics[width=12cm]{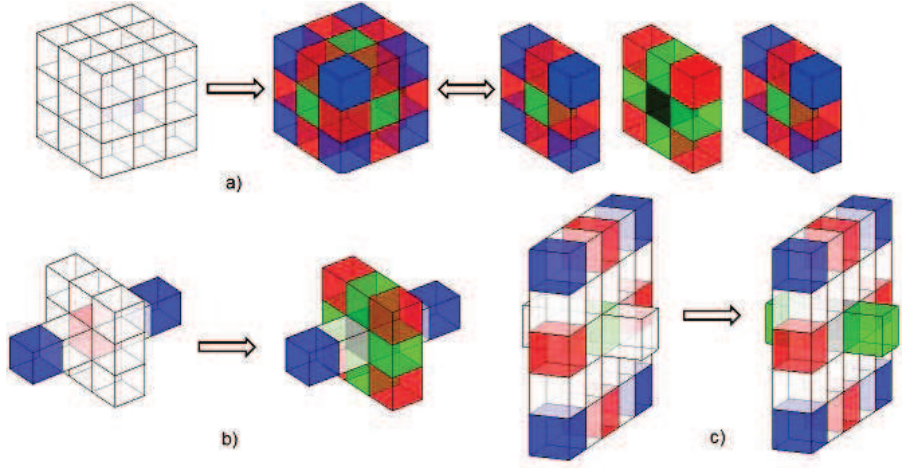}
	\caption{a), b) and c) Illustration of the three color-changing operations on $g_Q$ to obtain $g_P$.}
	\label{Op1}
\end{figure}

Fig.~\ref{Op1} illustrates the different color-changing operations described in 
Proc.~\ref{methodExtended}. In Fig.~\ref{caso2arreglado}.a, the 
color values of the ECM representation $E_{Q_1}$ of the 
cubical complex $Q_1$ represented on the top-left of Fig.~\ref{critical_conf}, are shown.
Fig.~\ref{caso2arreglado}.b shows the color values obtained after performing 
Proc.~\ref{methodExtended} on  $E_{Q_1}$.

\begin{figure}[t!]
	\centering
		\includegraphics[width=13cm]{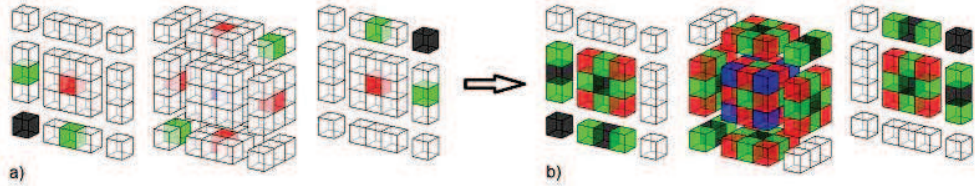}
	\caption{Color changing operation around a point representing the critical vertex in the top-left critical configuration of Fig.~\ref{critical_conf}.}
	\label{caso2arreglado}
\end{figure}

Now, we construct a new polyhedral complex $P(I)$ and a correspondence $h_P$ between points with color values $g_P(p)\neq -1$ and cells of  $P(I)$, in a  way that $g_P(p)=dim(h_P(p))$. Any point $p\in \mathbb{Z}^3$ with $g_P(p)=3$ will represent a polyhedron such that its faces will be represented by points in its neighborhood $N^{\leq 2}(p)$. Later, we will define a set of structuring elements $B_P$ for which $(g_P, h_P, B_P)$ is an ECM representation of  $P(I)$.

\begin{figure}[t!]
	\centering
		\includegraphics[width=13.5cm]{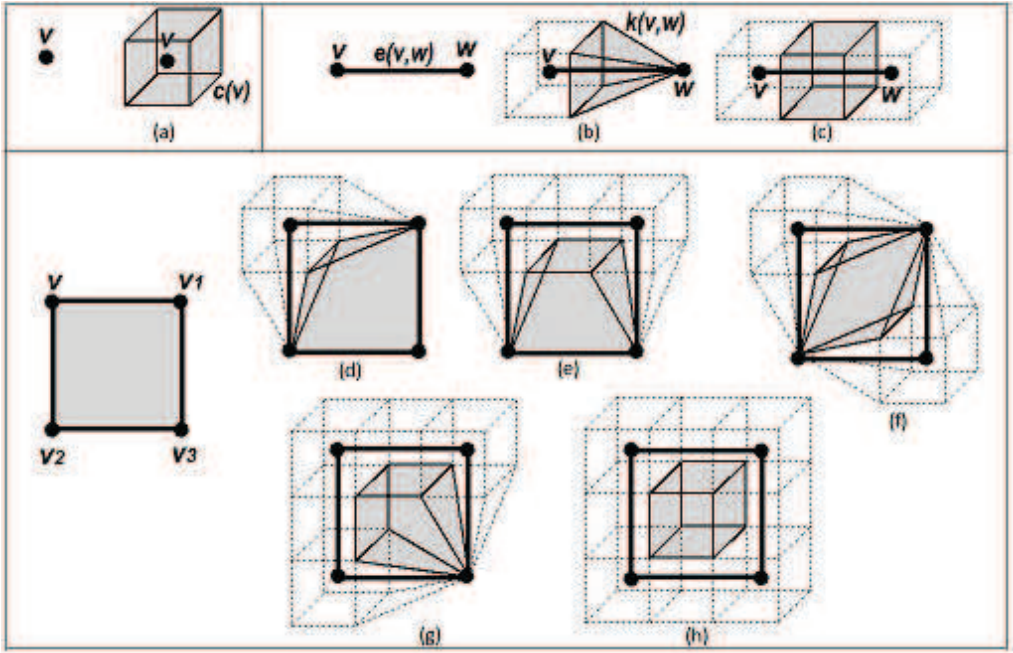}
	\caption{New polyhedra created in $P(I)$ associated with a vertex $v\in Q(I)$, a unit edge $e\in Q(I)$ with endpoints $v$ and $w$ and a unit square $s\in Q(I)$ with $0-$faces $v$, $v_1$, $v_2$ and $v_3$. }
	\label{poligonos5}
\end{figure}

The following polyhedra will be used in the construction of the polyhedral complex $P(I)$, that will be homotopy equivalent to $Q(I)$. 
\begin{itemize}
\item[(a)] The size$-\frac{1}{2}$ cube $c(v)$ centered at point $v\in \mathbb{R}^3$, with faces parallel to the coordinate planes.
See Fig.~\ref{poligonos5}.a.
\item[(b)] The pyramid $k(v,w)$  with apex a point $w$ and base the square-face of $c(v)$  whose barycenter lies on the edge $e(v,w)$ with endpoints $v$ and $w$.
 See Fig.~\ref{poligonos5}.b.
\item[(c)] The polyhedra $\{p_i(v, v_1,v_2,v_3)\}_{i=1,2,3,4}$ (where $v, v_1,v_2,v_3$ are four distinct points in $\mathbb{R}^3$ forming a unit square
$s$) given as follows:
\begin{itemize}
\item $p_1(v, v_1,v_2,v_3)$  is 
determined by the  triangles $t(v,v_1)$ and $t(v,v_2)$, and  the edges $e(v_1,v_3)$ and $e(v_2,v_3)$ 
where $t(v,v_i)$, $i=1,2$, is the triangle face of $k(v,v_i)$ whose barycenter lies on the square $s$. 
See Fig.~\ref{poligonos5}.d.
\item  $p_2(v, v_1,v_2,v_3)$
 is determined by the triangles $t(v,v_2)$ and $t(v_1,v_3)$, the
square $s(v,v_1)$  and  the edge $e(v_2,v_3)$,
where $s(v,v_1)$ is the square face of $c(\frac{v+v_1}{2})$
whose barycenter lies on the square $s$. 
See Fig.~\ref{poligonos5}.e.
\item  $p_3(v, v_1,v_2,v_3)$
 is determined by the four triangles $t(v,v_1)$,
 $t(v,v_2)$,
 $t(v_3,v_1)$ and
 $t(v_3,v_2)$. See Fig.~\ref{poligonos5}.f.
\item  $p_4(v, v_1,v_2,v_3)$
is determined  by the squares
 $s(v,v_1)$ and
 $s(v,v_2)$
and the triangles $t(v_1,v_3)$ and
 $t(v_2,v_3)$. See Fig.~\ref{poligonos5}.g.
\end{itemize}
\item The $22$ hexahedra $\{h_i(v_1,v_2,\ldots v_8)\}_{i=1,\ldots ,22}$ (where $v_1,v_2 \ldots , v_8\in \mathbb{R}^3 $ form a unit cube $c$ centered at a point $v\in \mathbb{Z}^3$, 
with faces parallel to the coordinate planes) showed in Fig.~\ref{hexahedron}. In this picture, each bolded point  represents the vertex of $c(v_i)$, for some $i=1,\dots,8$,
which lies inside the cube $c$. For example, the hexahedron on the top-left is $c$ and the one on the bottom-right is the size$-\frac{1}{2}$ cube $c(v)$. 
  \end{itemize}

\begin{figure}[t!]
	\centering
		\includegraphics[width=13cm]{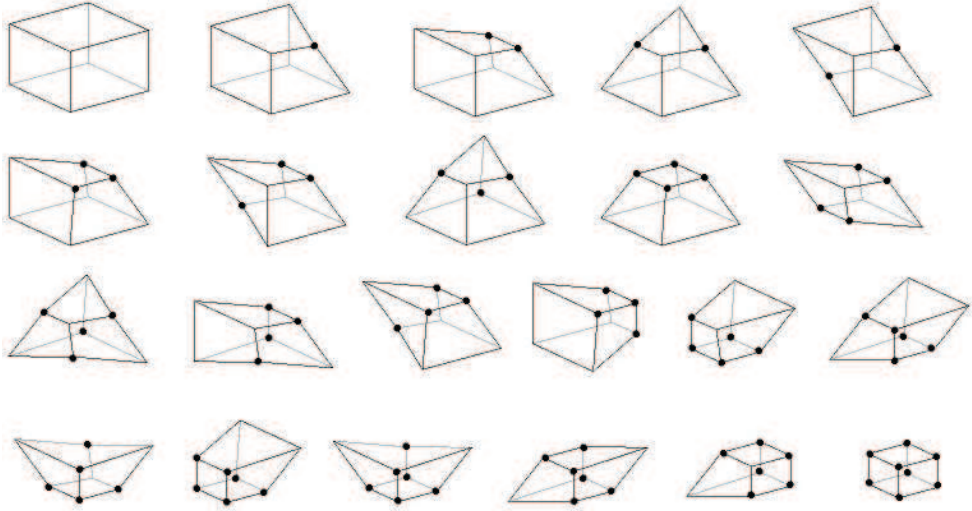}
	\caption{List of the  hexahedra that can appear in $P(I)$. Each 
  point in bold  represents the vertex of $c(v_i)$, for some $i=1,\dots,8$,
which lies inside the cube $c(v)$.}
	\label{hexahedron}
\end{figure}

The sets of data $\{E_{\alpha}=(h_{\alpha},g_{\alpha},B_{\alpha})\}_{\alpha=A,B,D,E,F,G}$ described below are ECM representations of the above polyhedra, whenever the coordinates of the vertices of the polyhedra are  multiple of $\frac{1}{2}$.
\begin{itemize}
\item[(a)] $D_A=N^{\leq 1}(p)$ for $p=4v$.

$g_A(p)=3$; 
  $g_A(q)=2$ if $q\in N^1_{6}(p)$;
 $g_A(q)=1$ if  $q\in N^1_{12}(p)$;  $g_A(q)=0$ if  $q\in N^1_{8}(p)$ 
and $g_A(q)=-1$ if $q\in\mathbb{Z}^3\setminus D_A$.
 See Fig.~\ref{nuevospolyhedroscubitos}.a. 

$h_A(p)=c(v)$;
and $h_A(q)=\sigma$ if $\sigma$ is a face of $c(v)$ and $q=4r_{\sigma}$.

Structuring elements 
 (modulo  $90-$degree rotations) 
 in $B_A$ are showed in Fig.~\ref{negro-verde}.a, 
Fig.~\ref{verde-rojo}.a     and    Fig.~\ref{rojo-azul}.a.  
\item[(b)] 
 $D_B= \{q\}\cup N^{\leq 1}(r)\setminus N^{1}(q)$ where  $q=4w$ and $r\in \mathbb{Z}^3$ is the closest point to $4r_{k(v,w)}$.

$g_B(r)=3$; $g_B(q)=0$;
$g_B(s)=2$ if $s\in N^1_{6}(r)\cap D_B$;
 $g_B(s)=1$ if  $s\in N^1_{12}(r)\cap D_B$;
    $g_B(s)=0$ if  $s\in N^1_{8}(r)\cap D_B$; 
 and $g_B(s)=-1$ if $s\in \mathbb{Z}^3\setminus D_B$.
See Fig.~\ref{nuevospolyhedroscubitos}.b.

$h_B(r)=k(v,w)$; $h_B(q)=w$;
and $h_B(s)=\sigma$ if  $\sigma$ is an $i$-face of $k(v,w)$, and $s\in D_B$ is the closest point to $4r_{\sigma}$ such that $g_B(s)=i$.

Structuring elements 
(modulo  $90-$degree  rotations) 
in $B_B$ are showed in 
Fig.~\ref{negro-verde}.b, 
Fig.~\ref{verde-rojo}.b     and    Fig.~\ref{rojo-azul}.(a,b). 
\item[(c)] The values of $g_{\alpha}$ for the rest of the ECM representation $E_{\alpha}$, $\alpha=D,E,F,G$, are showed, respectively, in  
 Fig.~\ref{nuevospolyhedroscubitos}.(d, e, f, g).

The values of $h_{\alpha}$ are defined in an analogous way as above.

 Structuring elements 
(modulo  $90-$degree  rotations) 
in $B_{\alpha}$, $\alpha=D,$ $E,F,G$,  are showed in Fig.~\ref{negro-verde}.c, 
Fig.~\ref{verde-rojo}.(c,d)     and    Fig.~\ref{rojo-azul}.(a,b).  
\end{itemize}

\begin{figure}[t!]
	\centering
		\includegraphics[width=12cm]{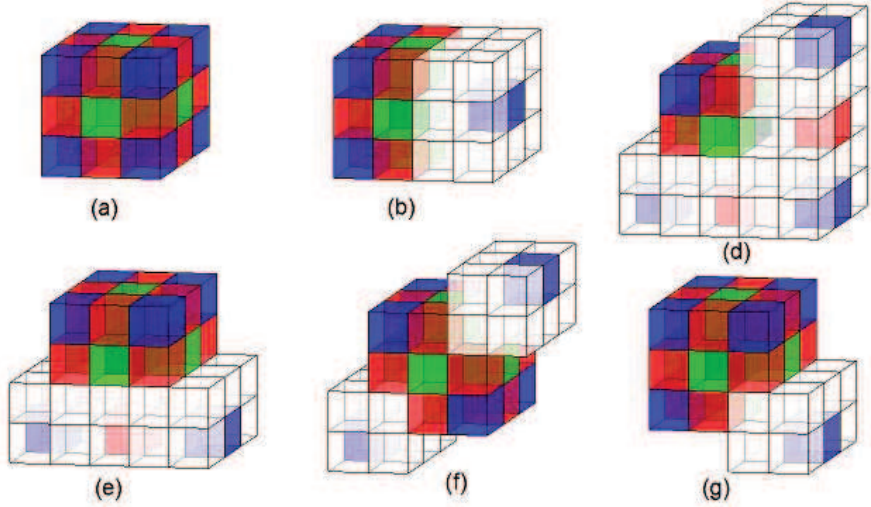}
	\caption{The color values in the ECM representations $E_{\alpha}$, $\alpha=A,B,D,E,F,G$, of the polyhedra (and all their faces) described in Fig.~\ref{poligonos5}.}  
	\label{nuevospolyhedroscubitos}
\end{figure}

The following procedure constructs the polyhedral complex $P(I)$ and the map  
 $h_{P} :D_P\to P(I)$.

\begin{proc}\label{method}
Let $E_Q$ be the ECM representation of $Q(I)$ from Prop.~\ref{EQ}. Let $g_P:\mathbb{Z}^3\to \mathbb{Z}$ the output of Proc.~\ref{methodExtended}.
Let $D_P=\{p\in \mathbb{Z}^3:g_P(p)\neq -1\}$. Let $R$ 
be the set of points in $\mathbb{Z}^3$ representing the critical vertices in $Q(I)$. Let 
 $ST_p$, $p\in R$, be the set of points in $\mathbb{Z}^3$
 representing the cells in $St\{h_Q(p)\}$ (these sets were computed  
 in  Proc.~\ref{methodExtended}).

\noindent Initially, $P(I)=Q(I)$, and $h_P(p)=h_Q(p)$ for any $p\in\mathbb{Z}^3$.

\noindent  For each $p\in R$,
\begin{itemize}
\item Replace $v=h_Q(p)\in P(I)$ by the cube $c(v)$ together with all its faces.
See Fig.~\ref{poligonos5}.a.
Define $h_{P}|_{D_A}=h_A$ where $D_A=N^{\leq 1}(p)$ and $h_A$ is described above.
\item For each $r\in ST_p$ such that $g_Q(r)=1$, let $v,w\in Q(I)$ represent the endpoints of $e(v,w)=h_Q(r)$. 
\begin{itemize}
\item If $w$  is non-critical:
\begin{itemize}
\item Replace $e(v,w)\in P(I)$ by  the pyramid $k(v,w)$ together with its 
triangular faces. See Fig.~\ref{poligonos5}.b;
\item Define $h_{P}|_{D_B}=
h_B$.
\end{itemize}
\item If $w$  is critical:
\begin{itemize}
\item Replace $e(v,w)\in P(I)$ by
  the  cube  $c(\frac{v+w}{2})$, together with its faces 
parallel to $e(v,w)$. See Fig.~\ref{poligonos5}.c.
\item Define $h_{P}|_{D_A}=h_A$ where $D_A=N^{\leq 1}(r)$.
\end{itemize}
\end{itemize}
\item For each $r\in ST_p$ such that $g_Q(r)=2$, let $v,v_1,v_2,v_3\in Q(I)$ represent the $0-$faces of $\sigma=h_Q(r)\in Q(I)$ (see Fig.~\ref{poligonos5}.left). Let $p_i=h^{-1}_Q(v_i)$ for $i=,2,3$.
\begin{itemize}
\item If $v_1,v_2,v_3$ are non-critical, 
\begin{itemize}
\item Replace $\sigma\in P(I)$ by
  the  polyhedron $p_1(v,v_1,v_2,v_3)$
together with its quadrangular faces. See Fig.~\ref{poligonos5}.d.
\item Define $h_{P}|_{D_D}=
h_D$.
\end{itemize}
\item If $v_1$ is critical but $v_2$ and $v_3$ are not, 
\begin{itemize}
\item Replace $\sigma\in P(I)$ by
  the    polyhedron $p_2(v,v_1,v_2,v_3)$
together with its quadrangular faces sharing the edge $e(v_2,v_3)$.
 See Fig.~\ref{poligonos5}.e.
\item Define $h_{P}|_{D_E}=
h_E$.
\end{itemize}
 An analogous case is when $v_2$ is critical but $v_1$ and $v_3$ are not.
\item If $v_3$ is critical but $v_1$ and $v_2$ are not, 
\begin{itemize}
\item Replace $\sigma\in P(I)$ by
  the    polyhedron $p_3(v,v_1,v_2,v_3)$
together with its quadrangular faces.
 See Fig.~\ref{poligonos5}.f.
\item  Define $h_{P}|_{D_F}=
h_F$.
\end{itemize}
\item If $v_1$ and $v_2$ are critical but $v_3$ is not, 
\begin{itemize}
\item Replace $\sigma\in P(I)$ by
  the polyhedron $p_4(v,v_1,v_2,v_3)$
together with its quadrangular faces sharing  $v_3$. See Fig.~\ref{poligonos5}.g.
\item  Define $h_{P}|_{D_G}=
h_G$.
\end{itemize}
\item If $v_1, v_2$ and  $v_3$ are critical, 
\begin{itemize}
\item Replace $\sigma\in P(I)$ by
  the  cube $c(\frac{v+v_1+v_2+v_3}{4})$
together with its square faces parallel to $\sigma$
 (see Fig.~\ref{poligonos5}.h.).
\item  Define $h_{P}|_{D_A}=
h_A$ where $D_A=N^{\leq 1}(r)$.
\end{itemize}
\end{itemize}
\item For each $r\in ST_p$ such that $g_Q(r)=3$, 
the $3$-cell $\sigma=h_Q(r)$ is replaced by a $3$-cell $\mu$ which is one of the hexahedra showed in Fig.~\ref{hexahedron} depending on the number and positions of the critical vertices that are faces of $\sigma$.  Define $h_P(r)=\mu$.
\end{itemize}
\end{proc}

\begin{theorem}\label{well-composedP}
$P(I)$ is a well-composed polyhedral complex homotopy equivalent to $Q(I)$.
\end{theorem}

\begin{proof}
Let us prove that the polyhedral complex $P(I)$ is well defined:
\begin{itemize}
\item By construction, all the faces of each cell in $P(I)$ are also in $P(I)$.
\item  
The intersection of any two cells of $P(I)$ is also a cell of $P(I)$, 
since during Proc.~\ref{method}, each critical vertex $v$ is replaced by $c(v)$ 
(together with all its faces)
in $P(I)$ and each cell in  $St \{v\}$ is replaced by a new polyhedron (together with its faces).
Besides, all the faces of the polyhedron added at each step are used in the next steps as faces of the new added polyhedra. 
Finally, an hexahedron of Fig.~\ref{hexahedron} substitutes the corresponding cube in $Q(I)$, whose faces are also faces of either previous cubes of $Q(I)$ or new polyhedra in $P(I)$.
\item $P(I)$ is complete  since the new added cells are always polyhedra together with all their faces.
\end{itemize}
Let us prove now that a homotopy equivalence (see \cite{h02})  from $P(I)$ to $Q(I)$ can be constructed.
The key point is that, for each critical vertex $v\in Q(I)$, the cube $c(v)\in P(I)$ is homotopy equivalent to ${v}$. The homotopy equivalence is given by the projection of all the points in $c(v)$ onto  $v$ (in fact, it is a deformation retraction).
The projection of the vertices of each cube $c(v)$ to  $v$ leads, in a natural way, to homotopy equivalences between the rest of the new constructed polyhedra in $P(I)$ and the corresponding cells of $Q(I)$:
\begin{itemize}
\item The pyramid $k(v,w)$
is  homotopy equivalent  to the edge $e(v,w)$, by the continuous function that maps the vertices of $s(v,w)$ 
(where $s(v,w)$ is the square face of $c(v)$ whose barycenter lies on $e(v,w)$)
  to $v$,  and $w$ to $w$; and it is extended continuously to all the points of the polyhedron;
\item  Analogously, the cube $c(\frac{v+w}{2})$ 
is  homotopy equivalent  to the edge $e(v,$ $w)$, by the continuous function that maps
$s(v,w)$  to $v$, and the ones of $s(w,v)$   to $w$.
\item  Each polyhedron $p_i(v,v_1,v_2,v_3)$  (see Fig.~\ref{poligonos5}) is homotopy equivalent to the corresponding square with vertices $v,v_1,v_2,v_3$ by the continuous function that maps the vertices of $c(w)$ to $w$ for $w=v,v_1,v_2,v_3$ if $w$ is critical, or $w$ to $w$ if not. It is extended continuously to the rest of the points of the polyhedron
\item Each hexahedron of quadrilateral faces (see Fig.~\ref{hexahedron}) is
homotopy
 equivalent to the corresponding unit cube obtained after 
mapping each size$-\frac{1}{2}$ cube $c(v)$ corresponding to a critical vertex $v\in Q(I)$, to $v$. 
\end{itemize}
Finally the polyhedral complex $P(I)$ is well-composed. Observe that after performing Proc.~\ref{method}, we have replaced  each critical vertex $v\in Q(I)$ by a cube $c(v)\subseteq P(I)$, and the set $St\{v\}\subseteq Q(I)$ by  $St c(v)\subseteq P(I)$ (see, for example, 
Fig.~\ref{critical_conf} and Fig.~\ref{stars}).
We have to prove that conditions $E1$ and $E2$ of Prop.~\ref{E1E2} are satisfied:
\begin{itemize}
\item[(E1)] Any edge $e=e(w_1,w_2)$ with endpoints $w_1$ and $w_2$ in $\partial P(I)$ has exactly two $2$-cofaces:
\begin{itemize}
\item If $e$ was also an edge in $\partial Q(I)$ satisfying $E1$, then neither $w_1$ nor $w_2$ were critical vertices and hence, $e$ will have as $2$-cofaces either square faces (that remain the same from $\partial Q(I)$) or quadrilateral faces of the polyhedra d) or e) of Fig.~\ref{poligonos5}.
\item If $e$ is an edge of a cube $c(v)$ for some critical vertex $v$, then $e$ has as $2$-cofaces a square face of the cube $c(v)$ and either a triangular face of a pyramid $k(v,w)$ for some non critical vertex $w$, or a square face of a cube $c(\frac{v+w}{2})$ for some critical vertex $w$.
\item If $e$ is an edge of a pyramid $k(v,w_2)$ for some critical vertex $v$, then either $e$ is shared by exactly two triangular faces of the pyramid, or it is shared by a triangular face of the pyramid and a quadrilateral face of a polyhedron $p_i(v,v_1, v_2,v_3)$ for some $i=1,2,3,4$, and $w_1=v_j$ for some $j=1,2,3$. See Fig.~\ref{poligonos5}.(d-g).  
\end{itemize}
\item[(E2)] For any vertex $w_1\in\partial P(I)$, $Lk\,\{w_1\}$ in $\partial P(I)$ has exactly one connected component: 
\begin{itemize}
\item If $w_1$ was also a vertex in $\partial Q(I)$ (that is, $w_1$ was a non-critical vertex and hence, satisfied condition $E2$), then $Lk\{w_1\}$ was a set of edges and vertices in $\partial Q(I)$, that are faces of $2-$cofaces of $w_1$. In the case that any of those $2-$cofaces were replaced by a polyhedron $p_i(v,v_1, v_2,v_3)$ for some $i=1,2,3,4$, some critical vertex $v$ and such that $w_1=v_j$ for some $j=1,2,3$, only one of the quadrilateral faces $f$ of  $p_i(v,v_1, v_2,v_3)$ would lie on $\partial P(I)$ and the edges and vertices in $Lk\{w_1\}$ in $\partial Q(I)$ would be replaced  by the edges and vertices of $c(v)$ and $k(v,w_1)$ shared with $f$. The way of construction of the new polyhedra guarantees that $Lk\{w_1\}$ in $\partial P(I)$ is still connected.
\item If $w_1$ is a vertex of a cube $c(v)$ for some critical vertex $v$ of $Q(I)$, then, by construction, the $2$-cofaces of $w_1$ are either square faces of $c(v)$, or triangular faces of a pyramid $k(v,w)$, for some non-critical vertex $w$, or square faces of an adjacent cube $c(\frac{v+w}{2})$ for some other critical vertex $w$, or quadrilateral faces of the polyhedra $p_i(v,v_1, v_2,v_3)$ showed in Fig.~\ref{poligonos5}. Then,  $Lk\,\{w_1\}$ in $\partial P(I)$, as one can guess from Fig.~\ref{poligonos5}, has exactly one connected component.
\end{itemize}
\end{itemize}
\end{proof}

\begin{figure}[t!]
	\centering
		\includegraphics[width=13.5cm]{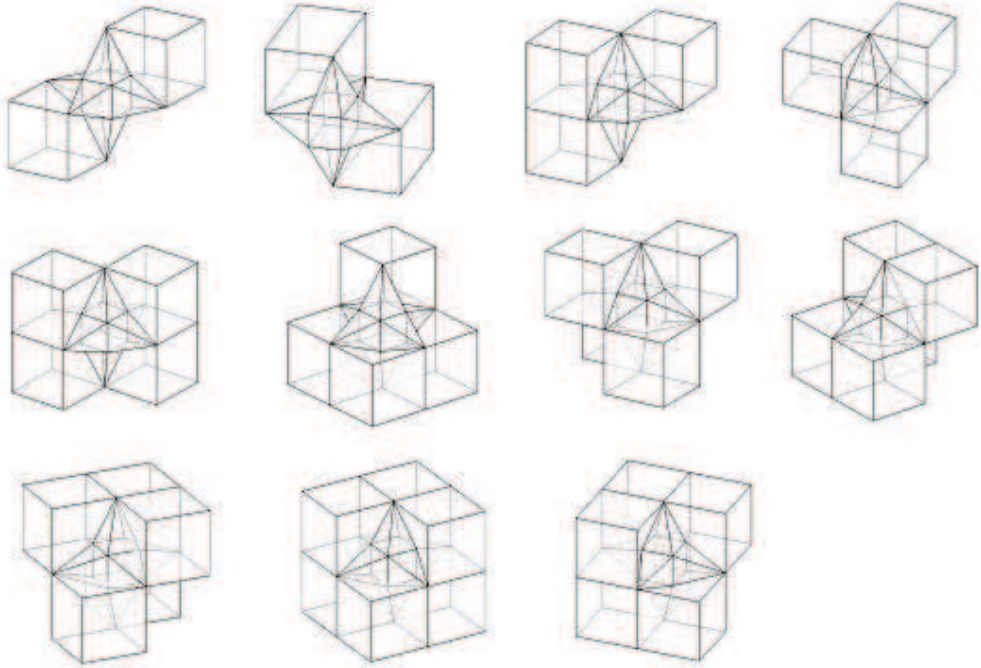}
	\caption{The well-composed polyhedral complexes obtained after performing Proc.~\ref{methodExtended} and Proc~\ref{method} \bf{only on the central vertex} for 
  the 3D complete cubical complexes showed in Fig.~\ref{critical_conf}.}
\label{stars}
\end{figure}

\begin{figure}[t!]
	\centering
		\includegraphics[width=13.5cm]{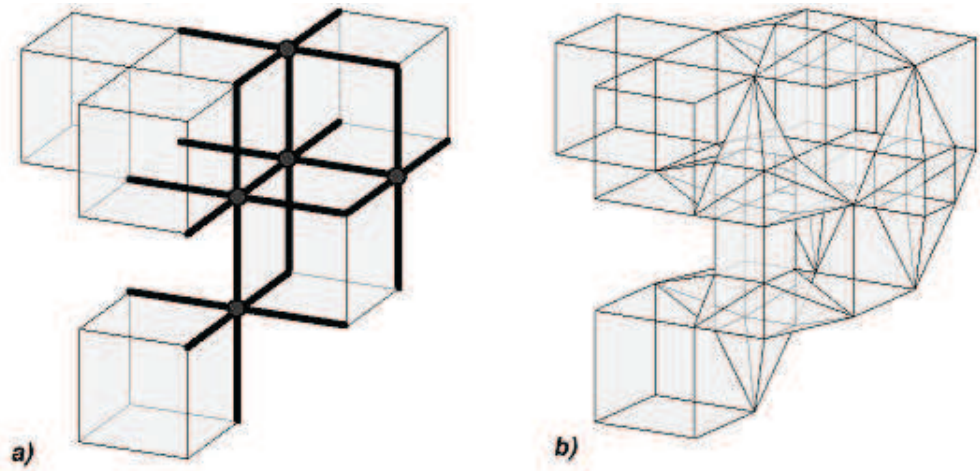}
	\caption{ a) A cubical complex $Q(I)$ associated to a 3D image composed by $6$ voxels.  b) The  obtained 3D well-composed polyhedral complex $P(I)$.}
	\label{ejemplo1}
\end{figure}

\begin{figure}[t!]
	\centering
		\includegraphics[width=13.5cm]{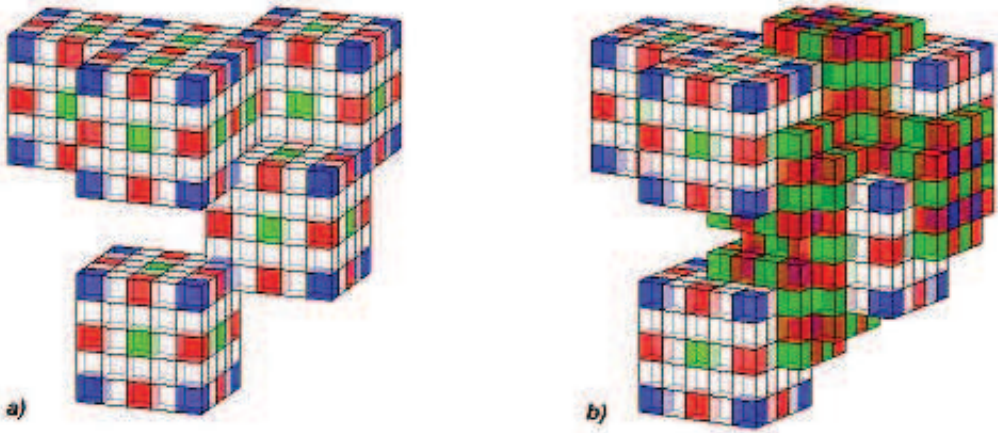}
	\caption{a)  The color values  in $E_Q$ where $Q(I)$ is the cubical complex showed in Fig.~\ref{ejemplo1}.a.
 b) The color values  in $E_P$ after performing the repairing process explained in Proc.~\ref{methodExtended}.}
	\label{ejemplo}
\end{figure}

In Fig.~\ref{ejemplo1}.a,  there is a simple example of a cubical complex $Q(I)$ associated to a 3D  image composed by $6$ voxels.  
The color values of the ECM representation $E_Q$ of $Q(I)$ is showed in Fig.~\ref{ejemplo}.a.  
The color values  in $E_P$ after performing the repairing process explained in Proc.~\ref{methodExtended}  is showed in 
Fig.~\ref{ejemplo}.b.  
In Fig.~\ref{ejemplo1}.b, there is the  3D
well-composed polyhedral complex $P(I)$ obtained following Proc.~\ref{method}. 

\begin{theorem} The triple $E_P=(h_P,g_P,B_P)$ is an ECM representation of $P(I)$.
\end{theorem}

\begin{proof}
Observe that  
$P(I)$ is constructed by replacing 
each cell in $Q(I)$ which is a coface of a critical vertex, by a particular polyhedron defined above. Observe that the color-values in the ECM representations $E_{\alpha}$, $\alpha=A,B,D,E$, $F$, $G$, of each new added polyhedron in $P(I)$, used for computing the ECM representation for $P(I)$,  coincides with the color-values computed in Proc.~\ref{methodExtended}.
For this reason, for any $p\in D_P$, $g_P(p)=dim(h_P(p))$.
  Finally, 
the new structuring elements  stored in $B_P$ are obtained from the ECM representation of each new added polyhedra in Proc.~\ref{method}.
\end{proof}

\section{Conclusion and Future Work}\label{conclusions}

In this paper, we have presented a representation scheme called ECM representation for storage and manipulate the cells of both an initial cubical complex $Q(I)$ associated to a 3D digital image $I$ and a 3D well-composed polyhedral complex homotopy equivalent to it.  
Using this scheme, we have presented a method for constructing the complex $P(I)$.

 Although the representation of the final polyhedron is $64=4^3$ times bigger than the input, plus
the space required to encode voxel colors, there are several advantages when using our approach as we will see below.
In order to compute properties on a polyhedral complex $K$  such as homology, we only need to store the map $g_K:\mathbb{Z}^3\to \mathbb{Z}$ and the structuring elements $B_K$. However, the set $B_P$ remain the same for any ECM representation $E_P$ of $P(I)$ obtained from a cubical complex $Q(I)$ after performing Proc.~\ref{method}. Therefore, Proc.~\ref{methodExtended} provides enough information to compute the homology of $P(I)$.


The overall method is linear in the number of cells of the initial cubical complex $Q(I)$. In fact, 
the method is linear in the number of critical vertices of $Q(I)$ and only the points in $N^{\leq 2}(p_v)$ (which represent the cells in 
$St\{v\}$ for any critical vertex $v$) are modified. 
Observe that the representation scheme could be improved, considering only the critical vertices of $Q(I)$ and their cofaces. 
Besides, we think that voxel colors may not be needed since the 
dimension of a cell $\sigma\in P(I)$ represented by a point $p\in \mathbb{Z}^3$ may be deduced from the coordinates of $p$.
We let these last tasks for future work.

Finally, this way of representation provides a fast access to the cells of $Q(I)$ and $P(I)$, in terms of the coordinates used to codify it, as well as an efficient way to get all the boundary faces of each cell. And what is more important, to obtain all this information we do not need to build $P(I)$.

Future work could be  to extend our method to  $n$D and to adapt the existing algorithms developed in the context of well-composed images to well-composed polyhedral complexes.

\noindent \textbf{Acknowledgments.} We wish to thank the anonymous referees for their helpful suggestions, which significantly improved the exposition.

\end{document}